\documentclass[11pt]{article}
\usepackage[margin=1in]{geometry}
\usepackage[round]{natbib}

\usepackage{palatino}

\usepackage{enumitem}
\usepackage[utf8]{inputenc}
\usepackage{amsmath,amssymb,amsthm}
\usepackage{tabu}
\usepackage{color}
\usepackage{calc}
\usepackage{tikz}
\usetikzlibrary{positioning,arrows,decorations.pathreplacing,shapes}
\usetikzlibrary{calc}
\usepackage{multirow}
\usepackage{boxedminipage}
\usepackage{xifthen}
\usepackage{tabularx}
\usepackage{hyperref}
\usepackage{algorithmic}
\usepackage[vlined,ruled]{algorithm2e}
\usepackage{thm-restate}

\newcommand{\R}{\mathbb{R}}

\newcommand{\x}{\mathbf{x}}
\renewcommand{\v}{\mathbf{v}}
\newcommand{\y}{\mathbf{y}}

\DeclareMathOperator{\E}{\mathbb{E}}
\DeclareMathOperator{\Prob}{\Pr}

\DeclareMathOperator{\regret}{\mathcal{R}}

\DeclareMathOperator{\diam}{diam}

\newtheorem{claim}{Claim}[section]

\newtheorem{proposition}{Proposition}

\newtheorem{theorem}{Theorem}

\newtheorem{definition}{Definition}
\newtheorem{lemma}{Lemma}

\usepackage{xcolor}
\definecolor{DarkGreen}{rgb}{0.1,0.5,0.1}

\newcommand{\footremember}[2]{%
	\footnote{#2}
	\newcounter{#1}
	\setcounter{#1}{\value{footnote}}%
}

\begin{document}

%

%

\title{Differentially Private Online Submodular Maximization}


\author{
	Sebastian Perez-Salazar\footremember{1}{Georgia Institute of Technology, \texttt{sperez@gatech.edu}}
	\and
	Rachel Cummings\footremember{2}{Georgia Institute of Technology, \texttt{rachelc@gatech.edu}. Supported in part by a Mozilla Research Grant, NSF grants CNS-1850187 and CNS-1942772 (CAREER), and a JPMorgan Chase Faculty Research Award.}
}

\maketitle

\begin{abstract}
  In this work we consider the problem of online submodular maximization under a cardinality constraint with differential privacy (DP). A stream of $T$ submodular functions over a common finite ground set $U$ arrives online, and at each time-step the decision maker must choose at most $k$ elements of $U$ before observing the function. The decision maker obtains a payoff equal to the function evaluated on the chosen set, and aims to learn a sequence of sets that achieves low expected regret. 
  
  In the full-information setting, we develop an $(\varepsilon,\delta)$-DP algorithm with expected $(1-1/e)$-regret bound of $\mathcal{O}\left( \frac{k^2\log |U|\sqrt{T \log k/\delta}}{\varepsilon}  \right)$. This algorithm contains $k$ ordered experts that learn the best marginal increments for each item over the whole time horizon while maintaining privacy of the functions. 
  In the bandit setting, we provide an $(\varepsilon,\delta+ O(e^{-T^{1/3}}))$-DP algorithm with expected $(1-1/e)$-regret bound of $\mathcal{O}\left( \frac{\sqrt{\log k/\delta}}{\varepsilon} (k (|U| \log |U|)^{1/3})^2 T^{2/3}  \right)$. 
  
  Our algorithms contains $k$ ordered experts that learn the best marginal item to select given the items chosen her predecessors, while maintaining privacy of the functions. One challenge for privacy in this setting is that the payoff and feedback of expert $i$ depends on the actions taken by her $i-1$ predecessors. This particular type of information leakage is not covered by post-processing, and new analysis is required.  Our techniques for maintaining privacy with feedforward may be of independent interest.
  
\end{abstract}


\section{Introduction}

Ensuring users' privacy has become a critical task in online learning algorithms. 
As an illustrative example, sponsored search engines aim to maximize the probability that displayed ads or products are clicked by incoming customers, but prospective customers do not want their privacy infringed after clicking on a product. Users visiting online retailer web-pages such as Amazon, Walmart or Target leave behind an abundance of sensitive personal information that can be use to predict their behaviors or preferences, potentially leading to catastrophic results~\citep{zhang2014privacy}\footnote{See also \url{https://www.nytimes.com/2012/02/19/magazine/shopping-habits.html}}.
In this work, we introduce the first algorithms for privacy-preserving online monotone submodular maximization under a cardinality constraint. 

A \emph{submodular} set function $f:2^U\to \mathbb{R}$ exhibits \emph{diminishing returns}, meaning that adding an element $x$ to a larger set $B$ creates less additional value than adding $x$ any subset of $B$.  (See Definition \ref{def.submod} in Section \ref{sec:prelim} for a formal definition.) Submodular functions have found widespread application in economics, computer science and operations research (see, e.g., \citet{bach2013learning} and \citet{krause2014submodular}), and have recently gained attention as a modeling tool for data summarization and ad display~\citep{ahmed2012fair, streetergolovinkrause2009online, badanidiyuru2014streaming}.  We additionally consider \emph{monotone} submodular functions, where adding elements to a set can only increase the value of $f$. Since unconstrained monotone submodular maximization is trivial---$f(S)$ can be maximized by choosing the entire universe $S=U$---we consider \emph{cardinality constrained} maximization, where the decision-maker solves: $\max_{S\subseteq U} f(S) \text{ s.t. } |S| \leq k$.



In the online learning setting, at each time-step $t$ a learner must choose a set $S_t\subseteq U$ of size at most $k$ and receives payoff $f_t(S_t)$ for a monotone submodular function $f_t$.  Importantly, the learner does not know $f_t$ before she chooses $S_t$, but this set can be chosen based on previous functions $f_1,\ldots,f_{t-1}$. Two types of informational feedback are commonly studied in the online learning literature. In the \emph{full-information} setting, the learner gets full oracle access to the function $f_t$ after choosing $S_t$, and thus is able to incorporate the entirety of previous functions into her future decisions. In the \emph{bandit} setting, the learner only observes her own payoff $f_t(S_t)$ as feedback.


Performance of an online learner is typically measured by the \emph{regret}, which is the difference between the best fixed decision in hindsight and the cumulative payoff obtained by the learner \citep{zinkevich2003online,hazan2016introduction,shalev2012online}. More precisely, the regret of a learner after $T$ rounds is: $\max_{|S|\leq k} \sum_{t=1}^T f_t(S) - \sum_{t=1}^T f_t(S_t)$. The aim often is to design algorithms with \emph{sublinear} regret, i.e., $o(T)$, so that the average payoff over time of the algorithm is comparable with the best average fixed profit in hindsight. Offline monotone submodular maximization under a cardinality constraint is NP-hard to approximate with a factor better than $(1-1/e)$~\citep{feige1998threshold, mirrokni2008tight}, so we instead measure the quality of our algorithms using the more restrictive notion of $(1-1/e)$-regret~\citep{streeter2009online,streetergolovinkrause2009online}:
\begin{equation}\label{eq.regret}
\regret_T = \left(1-\frac{1}{e}\right) \max_{|S|\leq k} \sum_{t=1}^T f_t(S) - \sum_{t=1}^T f_t(S_t).
\end{equation}


The privacy notion we consider in this work is \emph{differential privacy}~\citep{dwork2006calibrating}, which enables accurate estimation of population-level statistics while ensuring little can be learned about the individuals in the database. Informally, a randomized algorithm is said to be differentially private if changing a single entry in the input database results in only a small distributional change in the outputs. (See Definition \ref{def.dp} in Section~\ref{sec:prelim} for a formal definition.) This means that an adversary cannot information-theoretically infer whether or not a single individual participated in the database. Differentially private algorithms have been deployed by major organizations including Apple, Google, Microsoft, Uber, and the U.S. Census Bureau, and are seen as the gold standard in privacy-preserving data analysis. In this work, the input database to our learning algorithm consists of a stream of functions $F=\{f_1,\ldots,f_T \}$, and each individual's data corresponds to a function $f_t$.  Our privacy guarantees ensure that the stream of chosen sets $S_1, \ldots,S_T$ are differentially private with respect to this database of functions,


In both the full-information and bandit settings, we present differentially private online learning algorithms that achieve sublinear expected $(1-1/e)$-regret.

\paragraph{Motivating Example.} While there are countless examples of practical online submodular maximization problems using sensitive data, we offer this motivating example for concreteness. Consider an online product display model where a website has $k$ display slots and wants to maximize the probability of any displayed product being clicked. Each customer $t$ has a (privately known) probability $p_a^t$ of clicking a display for product $a\in U$, independently of the other products displayed.  Let $f_t(S)$ denote the probability that customer $t$ clicks on any product in a display set $S$.  We can write this function in closed form as $f_t(S)=1- \prod_{a\in S}(1-p_a^t)$. Note that this function is submodular because adding products to the set $S$ exhibits diminishing returns in total click probability. Each customer's click-probabilities $\{p^t_a\}_{a\in U}$ contain sensitive information about his preferences or habits, and require formal privacy protections.


\subsection{Our Results}
Our main results are differentially private algorithms for online submodular maximization under a cardinality constraint.  We provide algorithms that achieve sublinear expected $(1-1/e)$-regret in both the full-information and bandit settings.

Our algorithms are based on the approach of~\citet{streeter2009online}, who designed (non-private) online algorithms with low expected $(1-1/e)$-regret for submodular maximization. We adapt and extend their techniques to additionally satisfy differential privacy. Following the spirit of \citet{streeter2009online}, our algorithms have $k$ ordered online learning algorithms, or \emph{experts}, that together pick $k$ items at every time-step and learn from their decisions over time. Roughly speaking, expert $i$ learns how to choose an item that compliments the decisions of the previous $i-1$ experts. The expected $(1-1/e)$-regret can be bounded by the regret of these $k$ experts, so to show a low $(1-1/e)$-regret algorithm that preserves privacy, we simply need to find no-regret experts that together preserve privacy. Ideally, we would like each expert to be differentially private so that simple composition and post-processing arguments would yield overall privacy guarantees. Unfortunately this is not possible for $k>1$ because the choices of all previous experts alter the distribution of payoffs for expert $i$. 

Specifically, the $i$-th expert non-privately queries the function (i.e., accesses the database) at $|U|$ points that depend on the action of the previous experts.  A naive solution is to allow each expert to query the function at any of its $2^{|U|}$ values, and then privacy would be satisfied by post-processing on the differentially private outputs of previous experts. However, this larger domain size requires large quantities of noise that would harm the experts' no-regret guarantees. Effectively, this decouples the advice of the $k$ experts, so that experts are not learning from each other. This naturally helps privacy but harms learning. Instead, we restrict each expert to a domain of size $|U|$ that is defined by the actions of previous experts. This ensures no-regret learning, but post-processing no longer ensures privacy. We overcome this challenge by showing that \emph{together} the experts are differentially private and sufficiently low quantities of noise are needed.




Theorem \ref{thm:main} below is an informal version of our main results in the full-information setting (Theorems \ref{thm:privacy_full} and \ref{thm:full_info_regret} in Section \ref{sec:full_information}).

\begin{theorem}[Informal]\label{thm:main}
In the full-information setting, Algorithm~\ref{alg:EXP-ALG} for online monotone $k$-cardinality-constrained submodular maximization is $(\varepsilon,\delta)$-differentially private and guarantees
\begin{equation*}
\E\left[ \regret_T  \right] = \mathcal{O} \left( \frac{k^2 \log |U|\sqrt{T \log (k/\delta)}}{\varepsilon} \right).
	\end{equation*}
\end{theorem}


In the bandit setting, each expert only receives its own payoff as feedback, and does not have oracle access to the entire function.  For this setting, we modify the full-information algorithm by using a biased estimator of the marginal increments for other actions.

The algorithm also requires additional privacy considerations. The non-private approach of \citet{streeter2009online} randomly decides in each round whether to explore or exploit.  In exploit rounds, the experts sample a new set but play the current-optimal action, providing both learning and exploitation. Directly privatizing this algorithm incurs additional privacy loss from the exploit rounds, which leads to a weak bound of $\mathcal{O}(T^{3/4})$ for the expected $(1-1/e)$-regret, far from the best known $\mathcal{O}(T^{2/3})$. Instead, we have the experts sample new sets only after an exploration round has occurred. The choice to explore is data-independent, so privacy is maintained by post-processing.  If the exact number and timing of explore rounds are known in advance, this results in an $(\varepsilon,\delta)$-DP algorithm. However, this approach requires $\Omega(T^{2/3} + k|U|)$ space, which is not appealing in practical settings where $T$ is substantially larger than $U$. Instead we allow explore-exploit decisions to be made online and obtain a high probability bound on the number of explore rounds based on the sampling parameter.  At the expense of an exponentially small loss in the $\delta$ privacy parameter---resulting from the failure of the high probability bound---we obtain the asymptotically optimal $\mathcal{O}(T^{2/3})$ expected $(1-1/e)$-regret.


Theorem \ref{thm:main_bandit} is an informal version of our main results in the more challenging bandit feedback setting (Theorems \ref{thm:bandit_dp} and \ref{thm:bandit_regret} in Section \ref{sec:bandit}).


\begin{theorem}[Informal]\label{thm:main_bandit}
In the bandit feedback setting, Algorithm~\ref{alg:Interval-Bandit} for online monotone $k$-cardinality-constrained submodular maximization is $(\varepsilon,\delta+ e^{-8T^{1/3}})$-differentially private and guarantees
\begin{equation*}
	\E\left[\regret_T \right] = \mathcal{O}\left( \frac{\sqrt{\log k/\delta}}{\varepsilon} (k (|U| \log |U|)^{1/3})^2 T^{2/3}  \right).
\end{equation*}
\end{theorem}





The best known non-private expected $(1-1/e)$-regret in the full-information setting is $\mathcal{O}\left( \sqrt{k T  \log|U|  }\right)$ and in the bandit setting is $\mathcal{O}\left( k (|U|\log |U|)^{1/3} T^{2/3} \right)$~\citep{streeter2009online}. Comparing our expected $(1-1/e)$-regret bounds to these, we see that our bounds are asymptotically optimal in $T$, and have slight gaps in terms of $k$ and $U$. Typically, the dominating term is the time horizon $T$ with $k\leq |U| \ll T$, so our results match the best expected $(1-1/e)$-regret asymptotically in $T$.


Additionally, we show that our algorithms can be extended to a continuous generalization of submodular functions, know as \emph{DR-submodular} functions.  We provide a differentially private online learning algorithm for DR-submodular maximization that achieves low expected regret. A brief overview of this extension is given in Section \ref{sec:DR-submod}, with further details in the appendix.

\subsection{Related Work}\label{sec:review}



Online learning~\citep{zinkevich2003online,cesa2006prediction,hazan2016introduction,shalev2012online} has gained increasing attention for making decisions in dynamic environments when only partial information is available. Its applicability in ad placement~\citep{chatterjee2003modeling, chapelle2011empirical,tang2014ensemble} has made this model attractive from a practical viewpoint. 

Submodular optimization has been widely studied, due to the large number of important submodular functions, such as the cut of a graph, entropy of a set of random variables, and the rank of a matroid, to name only a few. For more applications see~\citep{schrijver2003combinatorial,williamson2011design,bach2013learning}. While (unconstrained) submodular minimization can be solved with polynomial number of oracle calls~\citep{schrijver2003combinatorial,bach2013learning}, submodular maximization is known to be NP-hard for general submodular functions. For monotone submodular functions under cardinality constraint, it is impossible to find a polynomial time algorithm that achieves a fraction better than $(1-1/e)$ of the optimal solution unless P=NP \citep{feige1998threshold}, and this approximation factor is achieved by the greedy algorithm \citep{fisher1978analysis}. For further results with more general constraints, we refer the reader to the survey \citep{krause2014submodular}. In the online setting, \citet{streeter2009online} and \citet{streetergolovinkrause2009online} were the first to study online monotone submodular maximization, respectively with cardinality/knapsack constraints and partition matroid constraints. Recently, continuous submodularity, has gained attention in the continuous optimization community~\cite{hassani2017gradient,niazadeh2018optimal,zhang2020one}. See~\cite{chen2018projection,chen2018online} for online continuous submodular optimization algorithms.

Differential privacy~\citep{dwork2006calibrating} has become the gold standard for individual privacy, and there as been a large literature developed of differentially private algorithms for a broad set of analysis tasks. See \citet{dwork2014algorithmic} for a textbook treatment. Due to privacy concerns in practical applications of online learning, there has been growing interest in implementing well-known methods---such as experts algorithms and gradient optimization methods--in a differentially private way. See for instance~\citep{jain2012differentially, thakurta2013nearly}. 

Differential privacy and submodularity were first jointly considered in \citep{gupta2010differentially}. They studied the combinatorial public projects problem, where the objective function was a sum of monotone submodular functions, each representing an agent's private valuation function, and a decision-maker must maximize this objective subject to a cardinality constraint.  The authors designed an $(\varepsilon,0)$-DP algorithm using the Exponential Mechanism of \citep{mcsherry2007mechanism} as a private subroutine, and achieved a $(1-1/e)$-approximation to the optimal non-private solution, plus an additional $\propto \varepsilon^{-1}$ term. Later, \cite{mitrovic2017differentially} extended these results to monotone submodular functions in the cardinality, matroid and $p$-system constraint cases. Their methods also used the Exponential Mechanism to ensure differential privacy.  See also recent work by~\cite{rafiey2020fast}. 

In the online learning framework, \citet{cardoso2019differentially} study online (unconstrained) differentially private submodular minimization. They use the Lov\'asz extension of a set function as a convex proxy to apply known privacy tools that work in online convex optimization~\citep{jain2012differentially, thakurta2013nearly}.  Since submodular minimization and maximization are fundamentally different technical problems, the techniques of \citet{cardoso2019differentially} do not extend to our setting.

Fundamental to our analysis is the differentially private Exponential Mechanism of \citet{mcsherry2007mechanism} and its inherent connection to multiplicative weights algorithms~\citep{hazan2016introduction,shalev2012online} to estimate probability distributions in the simplex while preserving privacy.

\section{Preliminaries}~\label{sec:prelim}
In this section we review definitions and properties of submodular functions and differential privacy.

\begin{definition}[Submodularity]\label{def.submod}
	A function $f:2^U\to \mathbb{R}$ is submodular if it satisfies the following diminishing returns property: For all $A\subseteq B\subseteq U$ and $x\notin B$,\footnote{ Equivalently, $f$ is submodular if $f(A\cap B) + f(A\cup B) \leq f(A) + f(B)$ for all $A,B\subseteq U$.}
	\[
	f(A\cup\{x\}) - f(A) \geq f(B\cup\{x\})- f(B).
	\]
\end{definition}

As is standard in the submodular maximization literature, we assume $f(\emptyset) = 0$. In our motivating example, this means that if no items are shown to the incoming customer, then the probability of selecting an item is $0$. We let $\mathcal{F}$ denote the family of submodular functions with finite ground set $U$. For the sake of simplicity, we will additionally assume that all functions take value in the interval $[0,1]$.  In this work, we additionally consider set functions $f$ that are monotone or non-decreasing, i.e., $f(A)\leq f(B)$ for all $A\subseteq B$.


In the problem of online monotone submodular maximization under a cardinality constraint, a sequence of $T$ monotone submodular functions $f_1,\ldots,f_T:2^U \to [0,1]$ arrive in an online fashion. At every time-step $t$, the decision maker $\mathcal{A}$ has to choose a subset $S_t\subseteq U$ of size at most $k$ before observing $f_t$. This decision must be based solely on previous observations. The decision maker $\mathcal{A}$ receives a payoff $f_t(S_t)$ and her goal is to minimize the $(1-1/e)$-expected-regret $\E[\regret_T]$, where $\regret_T =\left(1 - \frac{1}{e}\right) \max_{|S|\leq k} \sum_{t=1}^T f_t(S) - \sum_{t=1}^T f_t(S_t)$ as defined in Equation \eqref{eq.regret}, and the randomness is over the algorithm's choices.

A fundamental tool in our analysis is the Hedge algorithm (Algorithm~\ref{alg:Hedge}) of \citet{freund1997decision} which chooses an action from a set $[N]=\{1,\ldots,N\}$ based on past payoffs from each action. The algorithm takes as input a learning rate $\eta$ and a stream of linear functions $g_1,\ldots,g_T: [N]\to [0,1]$, where the payoff of playing action $i$ at time $t$ is $g_t(i)$.

  
In our setting, the learner must select a set of at most $k$ items from the ground set $U$. The learner does this by implementing $k$ ordered copies of the Hedge algorithm, each of which choses one item, so the action space for each instantiation is the ground set: $N=U$. The $i$-th copy of Hedge learns the item with the best marginal gain given the decisions made by the previous $i-1$ Hedge algorithms.


\begin{algorithm}[h]
	Initialize $w_1 = (1,\ldots,1)\in \mathbb{R}^N$\\
	\For{$t=1,\ldots,T$}{
		Sample action $i_t \in [N]$ w.p. $x_t(i)= \frac{w_t(i)}{\sum_j w_t(j)}$\\
		Obtain payoff $g_t(i_t)$ and full access to $g_t$\\
		Update $w_{t+1}(i)= w_t(i) e^{\eta g_t(i)}$
	}
	\caption{\textsc{Hedge}($\eta,g_1,\ldots,g_T$)}
	\label{alg:Hedge}
\end{algorithm}


The Hedge algorithm exhibits the following guarantee, which is useful for analyzing its regret, as well as the regret of our algorithms which instantiate Hedge.

\begin{theorem}[\cite{freund1997decision}]
	For any $i\in [N]$, the distributions $\x_1,\ldots,\x_T$ over $[N]$ constructed by Algorithm~\ref{alg:Hedge} satisfy
	\[
	\sum_{t=1}^T g_t(i)-\sum_{t=1}^T \x_t^\top g_t \leq \eta \sum_{t=1}^T \x_t^\top g_t^2 + \frac{\log N}{\eta},
	\]
where $g^2_t$ is the vector $g_t$ with each coordinate squared.
\end{theorem}


For the privacy considerations of this work, we view the input database as the ordered input sequence of submodular functions $F=\{ f_1,\ldots,f_T \}$ and the algorithm's output as the sequence of chosen sets $S_1,\ldots,S_T$.  We say that two sequences $F, F'$ of functions are \emph{neighboring} if $f_t\neq f_t'$ for at most one $t\in [T]$.

\begin{definition}[Differential Privacy \citep{dwork2006calibrating}]\label{def.dp}
	An online learning algorithm $\mathcal{A}:\mathcal{F}^T\to (2^{U})^T$ is $(\varepsilon,\delta)$-differentially private if for any neighboring function databases $F,F'$, and any event $S\subseteq (2^U)^T$,
	\[
	\Prob(\mathcal{A}(F)\in S ) \leq e^\varepsilon \Prob(\mathcal{A}(F')\in S) + \delta.
	\]
\end{definition}

Differential privacy is robust to post-processing, meaning that any function of a differentially private output maintains the same privacy guarantee.


\begin{proposition}[Post-Processing \citep{dwork2006calibrating}]
	Let $\mathcal{M}:\mathcal{F}^T \to \mathcal{R}$ be an $(\varepsilon,\delta)$-DP algorithm and let $h:\mathcal{R}\to \mathcal{R}'$ be an arbitrary function. Then, $\mathcal{M}' \doteq h\circ \mathcal{M}:\mathcal{F}^T\to \mathcal{R}'$ is also $(\varepsilon,\delta)$-DP.
\end{proposition}

Differentially private algorithms also \emph{compose}, and the privacy guarantees degrade gracefully as addition DP computations are performed. This enables modular algorithm design using simple differentially private building blocks. \emph{Basic Composition} \citep{dwork2006calibrating} says that can simply add up the privacy parameters used in an algorithm's subroutines to get the overall privacy guarantee.  The following Advanced Composition theorem provides even tighter bounds.


\begin{theorem}[Advanced Composition~\citep{dwork2010boosting}]
	Let $\mathcal{M}_1,\ldots,\mathcal{M}_k$ each be $(\varepsilon,\delta)$-DP algorithms. Then, $\mathcal{M}= (\mathcal{M}_1,\ldots,\mathcal{M}_k)$ is $(\varepsilon',k\delta+\delta')$-DP for $\varepsilon' = \sqrt{2k\log(1/\delta')}\varepsilon + k\varepsilon(e^\varepsilon-1)$ and any $\delta'\geq 0$.
\end{theorem}

Our algorithms rely on the Exponential Mechanism (EM) introduced by~\citet{mcsherry2007mechanism}. The EM takes in database $F$, a finite action set $U$, and a quality score $q:\mathcal{F}^T \times U \to \mathbb{R}$, where $q(F,i)$ assigns a numeric score to the quality of outputting $i$ on input database $F$.  The \emph{sensitivity} of the quality score, denoted $\Delta q$, is the maximum change in the value of $q$ across neighboring databases: $\Delta q = \max_{i \in U}\max_{F,F' \; neighbors} |q(F,i)-q(F',i)|$.  Given these inputs, the EM outputs $i\in U$ with probability proportional to $\exp(\varepsilon \frac{q(F,i)}{2\Delta q})$.  The Exponential Mechanism is $(\varepsilon,0)$-DP \citep{mcsherry2007mechanism}.


As noted by~\cite{jain2012differentially} and~\cite{dwork2010differential}, the Hedge algorithm can be converted into a DP algorithm using advanced composition and EM.

\begin{proposition}\label{prop:hedge_is_private}
If $\eta = \frac{\varepsilon}{\sqrt{32T \log 1/\delta}}$, then Hedge (Algorithm~\ref{alg:Hedge}) is $(\varepsilon,\delta)$-DP.
\end{proposition}


\section{Full Information Setting}\label{sec:full_information}

In this section, we introduce our first algorithm for online submodular maximization under cardinality constraint. It is both differentially private and achieves the best known expected $(1-1/e)$-regret in $T$. For cardinality $k$, the learner implements $k$ ordered copies of the Hedge algorithm. Each copy is in charge of learning the marginal gain that complements the choices of the previous Hedge algorithms. At time-step $t$, each Hedge algorithm selects an element $a\in U$ and the learner gathers these choices to play the corresponding set. When she obtains oracle access to the submodular function, for each $i \in [k]$, she constructs a vector $g_t^i$ with $a$-th coordinate given by the marginal gain of adding $a\in U$ to the choices made by the previous $i-1$ Hedge algorithms. Finally, she feeds back the vector $g_t^i$ to Hedge algorithm $i$. A formal description of this procedure is presented in Algorithm~\ref{alg:EXP-ALG}. 



\begin{algorithm}
	\textbf{Initialize:} Set $\eta= \frac{\varepsilon}{k\sqrt{32 T \log (k/\delta) }}$\\
	Instantiate $k$ parallel copies $\mathcal{E}_1,\ldots,\mathcal{E}_k$ of Hedge algorithm with rate $\eta$.\\
	\For{$t=1,\ldots,T$}{
		For each $i=1,\ldots,k$, sample $a_t^i$ given by $\mathcal{E}_i$.\\
		Play $S_t = \cup_{i=1}^k\{ a_t^i \}$. \\
		Obtain $f_t(S_t)$ and oracle access to $f_t$.\\
		For each $i=1,\ldots,k$, define linear function $g_t^i:U \to [0,1]$:
		\[
		g_t^i(a) = f_t(S_t^{i-1}+a)-f_t(S_t^{i-1}),\quad \forall a \in U,
		\]
		where $S_t^{i} = \cup_{j=1}^i \{ a_t^j \}$.\\
		Feed back each Hedge algorithm $\mathcal{E}_i$ with $g_t^i$
	}
	\caption{\textsc{FI-DP}$(F=\{ f_t \}_{t=1}^T,k,\varepsilon,\delta)$}
	\label{alg:EXP-ALG}
\end{algorithm}

To ensure differential privacy, it would be enough to show that each Hedge $\mathcal{E}_i$ is $(\varepsilon/k,\delta/k)$-DP. Indeed, if the sequence $(a_1^i,\ldots,a_T^i)$ constructed by each Hedge algorithm $i$ is $(\varepsilon/k,\delta/k)$-DP, then by Basic Composition and post-processing, the sequence $(S_1,\ldots,S_T)$ is $(\varepsilon,\delta)$-DP, where $S_t = \{ a_t^i \}_{i=1}^k$. However, for $i\geq 2$, the output of expert $\mathcal{E}_i$ depends on the choices made by algorithms $\mathcal{E}_1,\ldots,\mathcal{E}_{i-1}$. Moreover, algorithm $\mathcal{E}_i$ by itself is again accessing the database $F$, hence ruling out a post-processing argument. Despite this, we show that all experts together are $(\varepsilon,\delta)$-DP even though individually we cannot ensure they preserve $(\varepsilon/k,\delta/k)$-DP. 


It is worth noting that the Hedge algorithms $\mathcal{E}_1,\ldots,\mathcal{E}_k$ in Algorithm~\ref{alg:EXP-ALG} can be replaced by any other no-regret DP method that selects items over $U$, and the same proof structure would follow---although the regret bound would depend on the choice of no-regret algorithm. For instance, if we utilize the private experts method of\citep{thakurta2013nearly} instead of the Hedge algorithm, Algorithm~\ref{alg:EXP-ALG} would be $(\varepsilon,0)$-DP with a regret bound of $\mathcal{O}\left( k^2\frac{\sqrt{|U|T \log^{2.5} T}}{\varepsilon}\right)$.


\begin{theorem}\label{thm:privacy_full}
	Algorithm~\ref{alg:EXP-ALG} is $(\varepsilon,\delta)$-differentially private.
\end{theorem}

\begin{theorem}\label{thm:full_info_regret}
	Algorithm~\ref{alg:EXP-ALG} has $(1-1/e)$-expected-regret
	\begin{align*}
	\E\left[\regret_T\right] & \leq \mathcal{O}\left( \frac{k^2 \log |U| \sqrt{T \log (k/\delta)}}{\varepsilon}  \right).
	\end{align*}
\end{theorem}

\paragraph{Proof of Theorem~\ref{thm:privacy_full}}


The output of Algorithm~\ref{alg:EXP-ALG} is the stream of sets $(S_1,\ldots,S_T)$. Before showing that this output preserves privacy, we deal with a simpler case from which we can deduce an inductive argument.

Note that $\mathcal{E}_1(F)$ receives as feedback the functions $g_t^1 = (f_t(a))_{a\in U}$ at each time step. By Proposition~\ref{prop:hedge_is_private}, we have that $\mathcal{E}_1$ is $(\varepsilon/k,\delta/k)$-DP given that $\eta = \frac{\varepsilon}{k\sqrt{32T \log k/\delta}}$. On the other hand $\mathcal{E}_2(F)$ receives as feedback the functions $g_t^2 = ( f_t(a_t^1 + a) - f_t(a_t^1) )_{a\in U}$ at each time-step, where $a_t^1$ is computed by $\mathcal{E}_1(F)$. Therefore, the output of $\mathcal{E}_2$ depends uniquely on the choices of $\mathcal{E}_1$, hence, conditioning on these choices, $\mathcal{E}_2$ should also be $(\varepsilon/k,\delta/k)$-DP.  We generalize and formalize this in the next few paragraphs.

%

Consider the following family of algorithms: For $a^1,\ldots,a^{i-1}\in U^T$ let $S^{i-1}=\{a^{i-1},\ldots, a^{1}\}$. For $t=1,\ldots,T$, let $\mathcal{M}^{S^{i-1}}_t:\mathcal{F}^T\to \Delta(U)$ be the EM that outputs $a\in U$ with probability proportional to $e^{\eta \sum_{\tau < t} f_\tau (S_\tau^{i-1} \cup \{a\}) - f_\tau (S_\tau^{i-1})  }$.
Each of these mechanisms is $2\eta$-DP by Proposition \ref{prop:hedge_is_private}. Therefore, by Advanced Composition and our choice of $\eta$, $\mathcal{M}^{S^{i-1}} := (\mathcal{M}_1^{S^{i-1}},\ldots, \mathcal{M}_T^{S^{i-1}})$ is $(\varepsilon/k,\delta/k)$-DP. Note that for $S\subseteq U^T$ we have
\begin{align*}
	&\Prob( \mathcal{E}_i(F) \in S \mid (\mathcal{E}_{i-1},\ldots,\mathcal{E}_1)(F)= S^{i-1}  )  =  \Prob(\mathcal{M}^{S^{i-1}}(F) \in S  )
\end{align*}
and the latter expression describes the output of an $(\varepsilon/k,\delta/k)$-DP algorithm. This formalizes the idea that $\mathcal{E}_2$ is $(\varepsilon/k,\delta/k)$-DP if the choices of $\mathcal{E}_1$ are fixed. We utilize this idea to show that \emph{together} $(\mathcal{E}_k,\ldots,\mathcal{E}_1)$ are $(\varepsilon,\delta)$-DP. This is formally presented in Lemma~\ref{lem:composition_case_2}. The proof of this result (formally given in Appendix \ref{app.lem1}) is an inductive argument that takes advantage of the DP guarantee of the mechanisms $\mathcal{M}^{S^{i-1}}$.

\begin{restatable}{lemma}{composition}\label{lem:composition_case_2}
For any $i\in [k]$, the function $(\mathcal{E}_i,\mathcal{E}_{i-1},\ldots,\mathcal{E}_1):\mathcal{F}^T \to U^T \times \cdots \times U^T$ which is the composition of the first $i$ Hedge algorithms is $(i\varepsilon/k,i\delta/k)$-DP.
\end{restatable}

Lemma~\ref{lem:composition_case_2} with $i=k$ and post-processing ensures that Algorithm~\ref{alg:EXP-ALG} is $(\varepsilon,\delta)$-DP.

\paragraph{Proof of Theorem~\ref{thm:full_info_regret}}

The key idea is to bound the $(1-1/e)$-regret of Algorithm~\ref{alg:EXP-ALG} by the regret incurred by the $k$ Hedge algorithms $\mathcal{E}_1,\ldots,\mathcal{E}_k$. We formalize this in Proposition~\ref{prop:submod_regret_bound} below. With this bound, we can utilize the regret bound of the Hedge algorithm and conclude the proof. The regret incurred by $\mathcal{E}_i$ is
\begin{align*}
	r_i = \max_{a\in U}& \sum_{t=1}^T g_t^i(a) - \sum_{t=1}^T g_t^i(a_t),
\end{align*}
where $g_t^i(a) = f_t(S_t^{i-1}  \cup \{a\} ) - f_t(S_t^{i-1})$.

\begin{restatable}{proposition}{submodregret}\label{prop:submod_regret_bound}
	The $(1-1/e)$-regret of Algorithm~\ref{alg:EXP-ALG} is bounded by the expected regret of $\mathcal{E}_1,\ldots,\mathcal{E}_k$.
\end{restatable}

 While a full proof of Proposition \ref{prop:submod_regret_bound} is deferred to in Appendix \ref{app.prop3}, we describe the key idea here. To bound the $(1-1/e)$-regret, we rewrite the regret $r_i$ via the function $F:2^{[T]\times U}\to [0,1]$, $F(A)=\frac{1}{T}\sum_{t=1}^T f(A_t)$, where $A_t=\{ u\in U : (t,u)\in A \}$ as:
\[
\frac{r_i}{T} = \max_{a\in U} F(\widetilde{S}^{i-1}\cup\{a\}) - F(\widetilde{S}^{i})
\]
where $\widetilde{S}^{\ell}= \bigcup_{t=1}^T \{ t \} \times S^{\ell}$. We show that 
$F(\widetilde{S}^i) - F(\widetilde{S}^{i-1}) \geq \frac{F(\widetilde{OPT}) - F(\widetilde{S}^{i-1})}{k} - \frac{r_i}{T}$,
where $\widetilde{OPT}$ is the extension of $OPT=\mathrm{argmax}_{|S|\leq k} \sum_{t=1}^T f_t(S)$ to $[T]\times U$. Upon unrolling this recursion, we obtain the result. 

	To finish the proof of Theorem~\ref{thm:full_info_regret} we need to bound the overall regret of all $\mathcal{E}_i$. Observe that once we have fixed $S_1^{i-1},\ldots,S_T^{i-1}$, the feedback of expert $i$ is completely determined since the elements $a_t^1,\ldots,a_t^{i-1}$ depend only on experts $1,\ldots,i-1$. Therefore, we have
	\begin{align*}
		&\E[r_i\mid S_1^{i-1},\ldots,S_T^{i-1}] \leq \eta T + \frac{\log |U|}{\eta}
	\end{align*}
	by the Hedge regret guarantee. Integrating from $k$ to $1$ we get $\E\left[ \mathcal{R}_T \right] \leq \sum_{i=1}^k \E[r_i] \leq k \left( \eta T + \frac{\log|U|}{\eta} \right)$, and the result follows with our choice of $\eta = \frac{\varepsilon}{k\sqrt{32 T \log (k/\delta)}}$.


\section{Bandit Setting}\label{sec:bandit}

In the bandit case, the algorithm only receives as feedback the value $f_t(S_t)$. Given this restricted information, the algorithm must trade-off exploration of the function with exploiting current knowledge. As in~\citep{streeter2009online}, our algorithm controls this tradeoff using a parameter $\gamma \in [0,1]$, and by randomly exploring in each time-step independently with probability $\gamma$.

The non-private approach of \citet{streeter2009online} obtains $\mathcal{O}(T^{2/3})$ expected $(1-1/e)$-regret, and works as follows: In exploit rounds (prob. $1-\gamma$), play the experts' sampled choice $S_t$ and feed back $0$ to each $\mathcal{E}_i$. In explore rounds (prob. $\gamma$), select $i\in[k]$ and $a \in U$ uniformly at random. Play set $S_t = S_t^{i-1}+a$, observe feedback $f_t(S_t^{i-1}+a)$, give this value to $\mathcal{E}_i$, and feedback $0$ to the remaining experts. 

As we show in Appendix \ref{app.directregret}, directly privatizing this algorithm using the Hedge method from the full-information setting results in an expected $(1-1/e)$-regret of $\mathcal{O}(T^{3/4})$, which is far from the optimal $\mathcal{O}(T^{2/3})$. The problem with this naive approach is that a new sample is obtained via the Hedge algorithms at every time-step, including exploit steps, so to ensure $(\varepsilon,\delta)$-DP, a learning rate of $\eta= \frac{\varepsilon}{k\sqrt{32T\log (k/\delta)}}$ is required.

We improve upon this by calling the Hedge algorithm only after an exploration time-step has occurred, and new information is available. The learner continues playing this same set until the next exploration round, and privacy of these exploitation rounds follows from post-processing. This dramatically reduces the number of rounds that access the dataset, and reduces the overall amount of noise required for privacy. 


If the exact number of exploration rounds were known, this could be plugged into the learning rate $\eta$ to achieve $(\varepsilon,\delta)$-DP. In the non-private setting, a \emph{doubling trick} (see, e.g.,~\cite{shalev2012online}) can be employed to find the right learning rate by calling the algorithm multiple times, doubling $T$ and thus rescaling $\eta$ on each iteration.  Unfortunately, this doubling trick does not work in the private setting due to the direct non-linear connection between $\varepsilon$ the privacy parameter, $T$ the time horizon and $\eta$ the learning rate, as specified in Proposition \ref{prop:hedge_is_private}.  Instead we use concentration inequalities~\citep{alon2004probabilistic} to ensure that there are no more than $2\gamma T$ exploration rounds, except with probability $e^{-8{T^{1/3}}}$. With this, we can select a fixed learning rate $\eta = \frac{\varepsilon}{k\sqrt{32 (2\gamma T) \log (k/\delta)}}$ and guarantee optimal $\mathcal{O}(T^{2/3})$ expected $(1-1/e)$-regret, and the cost of $(\varepsilon,\delta+ e^{-8{T^{1/3}}})$-DP.


We remark in Appendix \ref{app.space} that this additional loss in the $\delta$ term can be avoided by pre-sampling the exploration round, but this requires $\Theta(T^{2/3} + k|U|)$ space, which may be unacceptable for large $T$.

Algorithm~\ref{alg:Interval-Bandit} presents the space-efficient approach. Here $\widehat{f}_t^i$ is the vector with $a$-th coordinate given by: 
$\widehat{f}_t^{i,a} = f_t(S_t^{i-1} +a ) \mathbf{1}_{\{ \text{Explore at time $t$, pick }i, \text{ pick }a  \}}$.



\begin{algorithm}
	\textbf{Initialize:} Set $\gamma = k \left(\frac{(16 |U| \log |U|)^{2}}{T}\right)^{1/3}$ and $\eta = \frac{\varepsilon}{k\sqrt{32 (2\gamma T) \log (k/\delta)}}$. \\
	Instantiate $k$ parallel copies $\mathcal{E}_1,\ldots,\mathcal{E}_k$ of Hedge algorithm with rate $\eta$.
	Utilize each $\mathcal{E}_i$ to sample $a_1^i$ and set $S_1=\{a_1^1,\ldots, a_1^k \}$.\\
	\For{$t=1,\ldots,T$}{
		Sample $b_t\sim \mathrm{Bernoulli}(\gamma)$.\\
		\If{$b_t=1$}{
			Sample $i\in [k]$ u.a.r. and $a\in U$ u.a.r.\\
			Play $S_t^{i-1}\cup\{a\}$.\\
			Obtain value $f_t(S_t)$. \\
			Feed back the function $\widehat{f}_t^i$ to expert $\mathcal{E}_i$, $\forall i$.\\
			Utilize $\mathcal{E}_i$ to pick $a_{t+1}^i$  $\forall i$.\\
			Update set $S_{t+1} = \cup_{i=1}^k \{ a_{t+1}^{i} \}$. 
		}
		\Else{
			Play $S_t $.\\
			Obtain $f_t(S_t)$. \\
			Update $S_{t+1}=S_t$.
		}
	}
	\caption{\textsc{BanditDP}$(F,\varepsilon,\delta)$}
	\label{alg:Interval-Bandit}
\end{algorithm}


\begin{theorem}\label{thm:bandit_dp}
	Algorithm~\ref{alg:Interval-Bandit} is $(\varepsilon,\delta+ e^{-8{T^{1/3}}})$-DP.
\end{theorem}

\begin{theorem}\label{thm:bandit_regret}
	Algorithm~\ref{alg:Interval-Bandit} has $(1-1/e)$-regret
	$$
	\E\left[\regret_T \right]   \leq \mathcal{O}\left( \frac{\sqrt{\log k/\delta}}{\varepsilon} (k (|U| \log |U|)^{1/3})^2 T^{2/3}  \right).
	$$
\end{theorem}

\paragraph{Proof of Theorem~\ref{thm:bandit_dp}}

%
	Observe that the algorithm only releases new information right an exploration time-step. If $t_1,\ldots,t_M$ are the exploration time-steps, with $M$ distributed as the sum of $T$ independent Bernoulli random variables with parameter $\gamma$, then conditioned on the event $M < 2\gamma T$, we know that the outputs $S_1, S_{t_1+1},\ldots,S_{t_M+1}$ are $(\varepsilon,\delta)$-DP by Theorem~\ref{thm:privacy_full}. Now, conditioning again on the event $M<2\gamma T$, the entire output $(S_1,\ldots,S_T)$ is $(\varepsilon,\delta)$-DP since this corresponds to post-processing over the previous output by extending the sets to exploitation time-steps. We know that $M \geq  2\gamma T$ occurs w.p. $\leq e^{-8\gamma^2 T}$. Thus, for any $S$ we have
	\begin{align*}
		&\Prob((\mathcal{E}_k,\ldots,\mathcal{E}_1 ) (F)\in S ) \\
		& \leq \Prob((\mathcal{E}_k,\ldots, \mathcal{E}_1)(F) \in S \mid M < 2\gamma T)\Prob(M < 2\gamma T) + e^{-8\gamma^2 T}\\
		&\leq e^\varepsilon\Prob((\mathcal{E}_k,\ldots,\mathcal{E}_1)(F')\in S \mid M < 2\gamma T )\Prob(M < 2\gamma T) + \delta + e^{-8 \gamma^2 T} \\
		& \leq e^{\varepsilon} \Prob((\mathcal{E}_k,\ldots, \mathcal{E}_1 )(F') \in S ) + \delta + e^{-8\gamma^2 T}.
	\end{align*}
	The result now follows by plugging in the value of $\gamma$ used in Algorithm \ref{alg:Interval-Bandit}.


\paragraph{Proof of Theorem~\ref{thm:bandit_regret}}


Theorem~\ref{thm:bandit_regret} requires the following two lemmas, proved respectively in Appendices \ref{app.lem2} and \ref{app.lem3}. The first lemma says that the $(1-1/e)$-regret experienced by the learner is bounded by the regret experienced by the expert and an additional error introduced during the exploration times. The second lemma bounds the regret experienced by the experts under the biased estimator.

\begin{restatable}{lemma}{banditone}\label{lem:bandit_lemma_1}
	If $r_i$ denotes the regret experience by expert $\mathcal{E}_i$ in Algorithm~\ref{alg:Interval-Bandit}, then
	\[
	\left( 1- \frac{1}{e} \right) \max_{|S|\leq k} \sum_{t=1}^T f_t(S) - \E\left[\sum_{t=1}^T f_t(S_t)\right] \leq  \sum_{i=1}^k \E[r_i] + \gamma T.
	\]
\end{restatable}

\begin{restatable}{lemma}{bandittwo}\label{lem:bandit_lemma_2}
	If each $\mathcal{E}_i$ is a Hedge algorithm with learning rate $\eta = \frac{\varepsilon}{k\sqrt{32 (2\gamma T) \log (k/\delta)}} $, then 
	\[ \E[r_i] \leq16 \frac{k^2 |U| \log|U| \sqrt{ T \log (k/\delta)}}{\varepsilon\sqrt{\gamma}} + \frac{k |U|}{\gamma } T \cdot e^{-8\gamma^2 T}. \]
\end{restatable}

Using these two results with $\gamma = k \left(\frac{(16 |U| \log |U|)^{2}}{T}\right)^{1/3}$:
\begin{align*}
	\E\left[\mathcal{R}_T \right]  &\leq k \left( 16 \frac{k^2 |U| \log|U| \sqrt{ T \log (k/\delta)}}{\varepsilon\sqrt{\gamma}}\right) + \frac{k |U|}{\gamma } T \cdot e^{-8\gamma^2 T}  + \gamma T \\
	& = \left( 16 \frac{k^3 |U| \log |U| \sqrt{\log k/\delta}}{\varepsilon} \sqrt{\frac{T}{\gamma}}  + \gamma T  \right)  + \frac{k|U|}{\gamma} T\cdot e^{-8\gamma^2 T}\\
	& \leq 32 \frac{\sqrt{\log k/\delta}}{\varepsilon} (k (|U| \log |U|)^{1/3})^2 T^{2/3} + \frac{|U|^{1/3} T^{4/3}}{(16 \log |U|)^{2/3}} e^{-8k^2(16|U|\log |U|)^{4/3} T^{1/3}}.
\end{align*}



\vspace{-0.3cm}
\section{Extension to Continuous Functions}\label{sec:DR-submod}

We sketch an extension of our methodology for (continuous) DR-submodular functions~\citep{hassani2017gradient,niazadeh2018optimal}. Further details can be found in Appendix \ref{app.continuous}.

Let $\mathcal{X}= \prod_{i=1}^n \mathcal{X}_i$, where each $\mathcal{X}_i$ is a closed convex set in $\R$. A function $f:\mathcal{X}\to \R_+$ is called \emph{DR-submodular} if $f$ is differentiable and $\nabla f(\x) \geq \nabla f(\y)$ for all $\x\leq \y$. DR-submodular functions do not fit completely in the context of convex functions. For instance, the multilinear extension of a submodular function~\citep{calinescu2011maximizing} is DR-submodular. The function $f$ is said to be $\beta$-smooth if $\| \nabla f(\x) - \nabla f(\y)\|_2 \leq \beta \|\x-\y\|_2$. In the online learning DR-submodular maximization problem, at each time-step $t=1,\ldots,T$, a $\beta$-smooth DR-submodular function $f_t:\mathcal{X} \to  [0,1]$ arrives and, without observing the function, the learner selects a point $\x_t \in \mathcal{X}$ learned using $f_1,\ldots,f_{t-1}$. She gets the value $f_t(\x_t)$ and also oracle access to $\nabla f_t$. The learner's goal is to minimize the $(1-1/e)$-regret $\mathcal{R}_T=\left( 1- \frac{1}{e} \right) \max_{\x\in \mathcal{P}} \sum_{t=1}^T f_t(\x) -  \sum_{t=1}^t f_t(\x_t)$. 

Online DR-submodular problems have been extensively studied in the full information setting---see for instance~\citep{chen2018online,chen2018projection,niazadeh2018optimal}. Similarly to the discrete submodular case, most of these methods implement $K$ ordered algorithms $\mathcal{E}_0,\ldots, \mathcal{E}_{K-1}$ for optimizing linear functions over $\mathcal{X}$. Algorithm $\mathcal{E}_k$ computes a direction of maximum increment from a point given by the algorithms $\mathcal{E}_{k-1},\ldots, \mathcal{E}_0$. The learner averages these directions to obtain a new point to play in the region $\mathcal{X}$. This is the continuous version of the Hedge approach. 

We show in Algorithm~\ref{alg:DR-submod} and Theorem \ref{thm:DR-submod} that a simple modification transforms the continuous method of \citet{chen2018online} into a differentially private one. For this, we utilize the Private Follow the Approximate Leader (PFTAL) framework of~\cite{thakurta2013nearly} as a black-box. PFTAL is an online convex optimization algorithm for minimizing $L$-Lipschitz convex functions over a compact convex region $\mathcal{X}$. In few words, their algorithm guarantees $(\varepsilon,0)$-DP and achieves an expected regret $\mathcal{O}\left( \frac{L^2\sqrt{n T \log^{2.5} T}}{\varepsilon} \right)$. 


\begin{algorithm}
	Let $K=\sqrt{\frac{\sqrt{T}}{\log^{2.5} T}}$. Initialize $\mathcal{E}_0,\ldots, \mathcal{E}_{K-1}$ parallel copies of PFTALs with privacy parameter $\varepsilon'=\varepsilon/K$.\\
	\For{$t=1,\ldots,T$}{
		\For{$k=0,\ldots,K-1$}{
			Let $\v_t^k$ be vector found using $\mathcal{E}_k$.
			}
		Let $\x_t = \frac{1}{K}\sum_{k=0}^{K-1} \v_t^k$.\\
		Play $\x_t $, receive $f_t(\x_t)$ and access to $\nabla f_t$.\\
		Feed back each $\mathcal{E}_k$ with the linear obsective $\ell_k(\v)=\nabla f_t( \x_t^k)^\top \v$ where $\x_t^k =\frac{1}{K} \sum_{i=0}^{k-1} \v_t^i$.
	}
	\caption{$(F=\{ f_t \}_{t=1}^T, \varepsilon)$}
	\label{alg:DR-submod}
\end{algorithm}

\begin{theorem}[Informal]\label{thm:DR-submod}
	Algorithm~\ref{alg:DR-submod} is $(\varepsilon,0)$-DP with expected $(1-1/e)$-regret \[\mathcal{O}\left( \frac{ T^{3/4}\sqrt{ \log^{2.5 }T }}{\varepsilon} \right).\]
\end{theorem}

The big $\mathcal{O}$ term hides dimension, bounds in gradient and diameter of $\mathcal{X}$ and only shows terms in $T$ and privacy parameter $\varepsilon$. The proof appears in Appendix~\ref{app.continuous}. 

\bibliography{biblio}

\appendix

\section{Appendix}


\subsection{Proof of Lemma~\ref{lem:composition_case_2}}\label{app.lem1}

\composition*

\begin{proof}

We prove the lemma by induction on $i$. The base case of $i=1$ follows from Proposition \ref{prop:hedge_is_private}. For the inductive step, assume the result is true for some $i\geq 1$, and we now prove that it also holds for $i+1$. That is, we aim to show that $(\mathcal{E}_{i+1},\ldots, \mathcal{E}_{1}):\mathcal{F}^T \to U^T\times \cdots \times U^T$ is $((i+1)\varepsilon',(i+1)\delta')$-private, where $\varepsilon' = \varepsilon/k$ and $\delta'=\delta/k$. Let $\wedge$ denote a maximum and recall that $\mathcal{M}^{S^{i}}$ is the behavior of the $i$-th expert across all $T$ rounds.


Consider the neighboring databases $F$ and $F'$. Pick any set $S\subseteq U$ and a fixed $S^{i}= (a^{i},\ldots,a^{1})\in(U^T)^i$, then
\begin{align*}
	&\Prob(\mathcal{E}_{i+1}(F) \in S \mid (\mathcal{E}_i,\ldots,\mathcal{E}_1)(F) = S^i ) \\
	& = \Prob(\mathcal{M}^{S^{i}} (F)  \in S ) \\
	& \leq (e^{\varepsilon'} \Prob(\mathcal{M}^{S^i}(F') \in S  ) )\wedge 1 + \delta' \tag{$(\varepsilon',\delta')$-DP of $\mathcal{M}^{S^i}$} \\
	& = (e^{\varepsilon'} \Prob( \mathcal{E}_{i+1}(F') \in S \mid (\mathcal{E}_i,\ldots,\mathcal{E}_1)(F') = S^i ) )\wedge 1 +\delta'.
\end{align*}
This is true as long as $(\mathcal{E}_i,\ldots,\mathcal{E}_1)(F) = S^i$ and $(\mathcal{E}_i,\ldots,\mathcal{E}_1)(F') = S^i$ are non-zero probability events, which is ensured to be true since the Hedge algorithm places positive probability on all events.

We can write
\[
\Prob((\mathcal{E}_i,\ldots,\mathcal{E}_1)(F) = S^i)  = e^{i\varepsilon'} \Prob((\mathcal{E}_i,\ldots,\mathcal{E}_1)(F') =S^i ) + \mu(S^i),
\]
where $\mu(S^i) = \Prob((\mathcal{E}_i,\ldots,\mathcal{E}_1)(F)=S^i) - e^{i\varepsilon'} \Prob((\mathcal{E}_i,\ldots,\mathcal{E}_1)(F') = S^i)$. We have $\mu(\mathcal{S})\leq i\delta'$ for any $\mathcal{S}\subseteq (U^T)^i$ since $(\mathcal{E}_i,\ldots,\mathcal{E}_1)$ is $(i\varepsilon',i\delta')$-DP by the inductive hypothesis.

Now, consider any set $\mathcal{S}\subseteq (U^T)^{i+1}$. Then,
\begin{align*}
	&\Prob((\mathcal{E}_{i+1},\mathcal{E}_i,\ldots,\mathcal{E}_1)(F)\in S) \\
	& = \sum_{S^i \in \mathcal{S}'} \Prob((\mathcal{E}_{i+1},S^{i})(F) \in S \mid (\mathcal{E}_i,\ldots,\mathcal{E}_1)(F) = S^{i}) \Prob((\mathcal{E}_i,\ldots,\mathcal{E}_1)(F)=S^i) \\
	& \leq \sum_{S^i \in \mathcal{S}'}  \left((e^{\varepsilon'} \Prob( (\mathcal{E}_{i+1},S^{i})(F') \in S\mid \mathcal{E}_1(F')=a^1) ) \wedge 1 + \delta'\right) \Prob((\mathcal{E}_i,\ldots,\mathcal{E}_1)(F)=S^i)\\
	& \leq \sum_{S^i \in \mathcal{S}'} \left((e^{\varepsilon'} \Prob( (\mathcal{E}_{i+1},S^{i})(F') \in S\mid (\mathcal{E}_i,\ldots,\mathcal{E}_1)(F')=S^i) ) \wedge 1\right)\left(e^{i\varepsilon'}\Prob({(\mathcal{E}_i,\ldots,\mathcal{E}_1)(F')=S^i}) + \mu(S^i)\right) \\
	& \qquad + \sum_{S^i \in \mathcal{S}'}\delta'\Prob((\mathcal{E}_i,\ldots,\mathcal{E}_1)(F)=S^i) \\
	& \leq e^{(i+1)\varepsilon'} \sum_{S^i \in \mathcal{S}'} \Prob((\mathcal{E}_{i+1},S^{i})(F')\in S \mid (\mathcal{E}_i,\ldots,\mathcal{E}_1)(F')=S^i) \Prob((\mathcal{E}_i,\ldots,\mathcal{E}_1)(F') = S^i)  + \mu(\mathcal{S}_+') +\delta'\\
	& \leq e^{(i+1)\varepsilon'} \Prob((\mathcal{E}_{i+1},\mathcal{E}_i,\ldots,\mathcal{E}_1)(F')\in S) + (i+1)\delta'
\end{align*}
where $\mathcal{S}' = \{ S^{i}\in (U^T)^i : (a^{i+1}, S^{i} )\in \mathcal{S} \text{ for some }  a^i\in U^T\}$ and $\mathcal{S}_{+}'$ are the elements $S^i\in \mathcal{S}'$ such that $\mu(\mathcal{S}')\geq 0$. This concludes the proof.

\end{proof}

\subsection{Proof of Proposition~\ref{prop:submod_regret_bound}}\label{app.prop3}

\submodregret*

\begin{proof}
	Fix the choices $S_1,\ldots,S_T$ of the experts arbitrarily, and let $r_i$ the overall regret experience by $\mathcal{E}_i$. That is,
	\begin{align*}
		r_i = \max_{a\in U}& \sum_{t=1}^T f_t(S_t^{i-1}+a) - f_t(S_t^{i-1})  - \sum_{t=1}^T f_t(S_t^{i-1}+a_t^i) - f_t(S_t^{i-1}).
	\end{align*}
	Define the new function $F:2^{[T]\times U} \to \R$ as
	\[
	F(A) = \frac{1}{T}\sum_{t=1}^T f_t(A_t),
	\]
	where $A_t = \{ x\in U : (t,x)\in A \}$. Clearly, $F$ is submodular, nondecreasing and $F(\emptyset)=0$. Then,
	\[
	\frac{r_i}{T} = \max_{a\in U} F(\widetilde{S}^{i-1}+\widetilde{a}) - F(\widetilde{S}^{i-1}) - (F(\widetilde{S}^{i}) - F(\widetilde{S}^{i-1})),
	\]
	where $\widetilde{S}^i = \bigcup_{t=1}^T \{t\}\times S^i$. 
	
	Let $OPT\subseteq U$ be the optimal solution of $\max_{|S|\leq k} \sum_{t=1}^T f_t(S)$ and consider its extension to $[T]\times U$, i.e., $\widetilde{OPT} = \bigcup_{t=1}^T \{t \}\times OPT$.
	
	\begin{claim}\label{claimA1}
		For any $i=1,\ldots,k$, $\max_{a\in U} F(\widetilde{S}^{i-1} + \widetilde{a}) - F(\widetilde{S}^{i-1}) \geq \frac{F(\widetilde{OPT}) - F(S^{i-1}) }{k}$.
	\end{claim}
	\begin{proof}[Proof of Claim \ref{claimA1}] 
		\begin{align*}
			&F(\widetilde{OPT}) - F(\widetilde{S}^{i-1}) \\
			& \leq  F(\widetilde{S}^{i-1}+\widetilde{OPT}) - F(\widetilde{S}^{i-1}) \\
			& \leq \sum_{\widetilde{a} \in \widetilde{OPT}\setminus \widetilde{S}^{i-1}} F(\widetilde{S}^{i-1}+\widetilde{a}) - F(\widetilde{S}^{i-1})\\
			& \leq k \cdot \left(\max_{a\in U} F(\widetilde{S}^{i-1}+\widetilde{a}) - F(\widetilde{S}^{i-1})\right).
		\end{align*}
	\end{proof}
	
	Using this claim, we can see,
	\[
	F(\widetilde{S}^i) - F(\widetilde{S}^{i-1}) \geq \frac{F(\widetilde{OPT}) - F(\widetilde{S}^{i-1})}{k} - \frac{r_i}{T}.
	\]
	Unrolling the recursion, we obtain
	\[
	\sum_{t=1}^T f_t(S_t) \geq \left( 1- \frac{1}{e} \right) \sum_{t=1}^T f_t(OPT) - \sum_{i=1}^k r_i.
	\]
\end{proof}

\subsection{Proof of Lemma~\ref{lem:bandit_lemma_1}}\label{app.lem2}

\banditone*

\begin{proof}
	Observe that at exploration time-steps $\tau$, i.e, when $b_\tau=1$, Algorithm~\ref{alg:Interval-Bandit} plays a set of the form $S_\tau= S_\tau^{i-1}+a$. Right after this, the algorithm samples a new set $S_{\tau+1}$ given by the Hedge algorithms and will keep playing this set until the next exploration time step.
	
	For the sake of analysis, we introduce the following set. Let $t_0=0, t_1, \ldots, t_M$ be the times when a new sample set for exploitation is obtained. Note that besides time $t_0$, all times $t_1,\ldots,t_M$ are exploration times. Now, let $S_t'= S_{t_i}$ for $t=t_{i}+1,\ldots, t_{i+1}$. Note that for times $b_t=0$, then $S_t'=S_t$; however, for times $b_t=1$, then $S_t'$ is not necessarily the same as $S_t=S_t^{i-1}+a$. In other words, $S_t'$ corresponds to the real full exploitation scheme. Now, as in the full information setting, we have
	\[
	\left( 1- \frac{1}{e} \right) \max_{|S|\leq k} \sum_{t=1}^T f_t(S) - \sum_{t=1}^T f_t(S_t') \leq \sum_{i=1}^k r_i,
	\]
	where $r_i = \max_{a\in U} \sum_{t=1}^T f_t^{i,a} - \sum_{t=1}^T f_{t}^{i,a_t^i}$. Thus
	\begin{align*}
		&\left( 1- \frac{1}{e} \right) \max_{|S|\leq k} \sum_{t=1}^T f_t(S) - \E\left[\sum_{t=1}^T f_t(S_t)\right]\\
		&\leq \sum_{i=1}^k \E[r_i] + \E\left[\sum_{t=1}^T f_t(S_t')- f_t(S_t)\right] \\
		&\leq \sum_{i=1}^k \E[r_i] + \gamma T,
	\end{align*}
	since at the end, only the exploration times could contribute to the difference $f_t(S_t')-f_t(S_t)$ and those are $\gamma T$ in expectation.
\end{proof}

\subsection{Proof of Lemma~\ref{lem:bandit_lemma_2}}\label{app.lem3}

\bandittwo*

\begin{proof}
	From the perspective of expert $\mathcal{E}_i$, at every time-step $t$, she sees the vector $\widehat{f}_t^i$ such that
	\[
	\widehat{f}_t^{i,a} = f_t(S_t^{i-1} +a ) \mathbf{1}_{\{ \text{Explore at time $t$, pick }i, \text{ pick }a  \}} 
	\]
	in its $a$-th coordinate. Notice that this vector is $0$ if no exploration occurs at time $t$. The expert $\mathcal{E}_i$ samples a new element in $U$ only after exploitation times. Observe that the feedback of $\mathcal{E}_i$ is independent of choices made by $\mathcal{E}_i$. Indeed, this feedback depends only on the set $S_t^{i-1}$ constructed by $\mathcal{E}_1,\ldots,\mathcal{E}_{i-1}$ and the decision of the learner to explore, which is independent of the learning task. Therefore, the sequence $\widehat{f}^i = (\widehat{f}_1^i,\ldots,\widehat{f}_T^i)$ could be considered oblivious for $\mathcal{E}_i$ and we can apply the guarantee of Hedge over $\widehat{f}_i$. That is, for any $a\in U$,
	\[
	\sum_{t=1}^T \widehat{f}_t^{i,a} - \sum_{t=1}^T \x_t^\top \widehat{f}_t^i \leq \eta\sum_{t=1}^T \x_t^\top (\widehat{f}_t^{i})^2 + \frac{\log |U|}{\eta},
	\]
where $\x_t\in \Delta(U)$ is the non-zero distribution used by expert $\mathcal{E}_i$ in the Hedge algorithm and $\Delta(U)= \{ \x \in \R^U : \|\x\|_1=1, \x\geq 0 \}$ is the probability simplex over elements in $U$. Notice that exploitation times appear in the summation with $0$ contribution. This expression is not the same as the regret of $\mathcal{E}_i$ but we can relate these quantities as follows. Conditioned on $S_{1}^{i-1},\ldots,S_T^{i-1}$ we obtain,
	\begin{align*}
		\E[\widehat{f}_t^{i,a}\mid S_{1}^{i-1},\ldots,S_T^{i-1}] = \frac{\gamma}{k|U|} f_t^{i,a} + \delta_t,
	\end{align*}
	where $f_t^{i,a}= f(S_t^{i-1}+a)-f(S_t^{i-1})$ and $\delta_t^i = \frac{\gamma}{k|U|} f(S_t^{i-1})$. Notice that $S_{t}^{i-1},\ldots,S_T^{i-1}$ are independent of actions taken by $\mathcal{E}_i$, so
\[
		\E[ \x_t^\top \widehat{f}_t^i\mid S_{1}^{i-1},\ldots,S_T^{i-1} ] = \frac{\gamma}{k|U|} \E[\x_t^\top f_t^{i} \mid S_t^{1},\ldots,S_T^{i-1} ]  + \delta_t
\]
	and
	\begin{align*}
	\E[\x_t^\top (\widehat{f}_t^i)^2\mid S_{1}^{i-1},\ldots,S_T^{i-1} ] &= \E\left[ \sum_{a\in U} x_t(a) (\widehat{f}_t^{i,a})^2 \mid S_{1}^{i-1},\ldots,S_T^{i-1}  \right]\\
		& = \sum_{a\in U} \E[x_t(a)\mid S_{1}^{i-1},\ldots,S_T^{i-1}] \frac{\gamma}{k|U|} f(S_t^{i-1}+a)^2 \\
		&\leq \frac{\gamma}{k|U|}.
	\end{align*}
	Let $M$ be the number of times Algorithm~\ref{alg:Interval-Bandit} decides to explore. That is, $M$ is distributed as the sum of $T$ Bernoulli random variables with parameter $\gamma$. By concentration bounds,
	\[
	\Prob(M> 2\gamma T) \leq e^{-8\gamma^2 T}.
	\]
	Now, let $ t_1,\ldots, t_M$ be the times the algorithm decides to explore and let $t_0 = 0$. For $i=1,\ldots,M$, we can assume that expert $\mathcal{E}_i$ releases the same vector $\x_t\in \Delta_U$ during the time interval $[t_{i-1},t_i)$ since she does not get any feedback during those times. If we consider $\eta = \frac{\varepsilon}{k\sqrt{32 (2\gamma T) \log (k/\delta)}}$, then for any $a\in U$ we have
	\begin{align*}
		\frac{\gamma}{k|U|}\E\left[ \sum_{t=1}^T f_t^{i,a} -  \sum_{t=1}^T \x_t^\top f_t^{i}  \right] &=\E\left[ \sum_{t=1}^T \widehat{f}_t^{i,a} - \sum_{t=1}^T \x_t^\top \widehat{f}_t^i \right] \\
		& \leq \left(\eta\sum_{t=1}^T \E\left[ \x_t^\top (\widehat{f}_t^{i})^2\right] + \frac{\log |U|}{\eta}\right)  + T\cdot e^{-8 \gamma^2 T} \\
		& \leq \left( \eta \frac{\gamma}{k|U|} T + \frac{\log|U|}{\eta}  \right) + T\cdot e^{-8\gamma^2 T}\\
	\end{align*}
	Therefore,
\[
		\E[r_i] = \max_{a\in U}\sum_{t=1}^{T} f_t^{i,a} - \E \left[\sum_{t=1}^T \x_t^\top f_t^i\right] \leq 16 \frac{k^2 |U| \log|U| \sqrt{ T \log (k/\delta)}}{\varepsilon\sqrt{\gamma}} + \frac{k |U|}{\gamma } T \cdot e^{-8\gamma^2 T}.
\]
\end{proof}

\section{Additional Results in Bandit Setting}

\subsection{$\mathcal{O}(T^{3/4})$ Regret Bound of Direct Approach in Bandit Setting}\label{app.directregret}

In the bandit setting, the direct approach for differential privacy corresponds to sampling a new set from the Hedge algorithms at each time step. As in the full-information setting, to ensure $(\varepsilon,\delta)$-DP, a learning rate of $\eta = \frac{\varepsilon}{k\sqrt{ 32 T \log (k/\delta)}}$ is enough. 

Similar to Lemma~\ref{lem:bandit_lemma_2}, in this setting we have
\[
\left( 1 - \frac{1}{e} \right)\max_{|S|\leq k} \sum_{t=1}^T f_t(S) - \E\left[ \sum_{t=1}^T f_t(S_t) \right] \leq \sum_{t=1}^k \E[r_i] + \gamma T.
\]
Since,
\begin{align*}
	\E[r_i] & \leq \frac{k|U|}{\gamma} \left( \eta \frac{\gamma}{k|U|} T + \frac{\log|U|}{\eta}  \right) \\
	& = \frac{k^3 |U| \sqrt{32 T \log(k\delta)}}{\varepsilon \gamma} + \frac{\varepsilon k\sqrt{T}}{\sqrt{32 \log (k/\delta)}},
\end{align*}
then we have,
\begin{align*}
	\left( 1 - \frac{1}{e} \right)\max_{|S|\leq k} \sum_{t=1}^T f_t(S) - \E\left[ \sum_{t=1}^T f_t(S_t) \right] &\leq \frac{k^4 |U| \sqrt{32 T \log(k\delta)}}{\varepsilon \gamma} + \frac{\varepsilon k^2\sqrt{T}}{\sqrt{32 \log (k/\delta)}} + \gamma T.
\end{align*}
This last bound is minimized when $\gamma= \Theta (T^{-1/4})$ which gives a $(1-1/e)$-regret bound of $\mathcal{O}(T^{3/4})$.

\subsection{Trading Off Privacy $\delta$-Term and Space}\label{app.space}

In this subsection, we show how to trade-off the $\delta$-term $e^{-8T^{1/3}}$ by allowing additional space. For each $t\in T$, select $t$ as an explore round independently with probability $\gamma$. Let $M$ be the number of time-steps selected. Note that $\E[M]=\gamma T$. Now, run Algorithm~\ref{alg:Interval-Bandit} with $\eta = \frac{\varepsilon}{k\sqrt{32(M+1) \log(k/\delta)}}$ and force the algorithm to explore at the $M$ sampled time-steps and utilize the rest of the time-steps to exploit. 

In this case, and following the proof of Lemma~\ref{lem:bandit_lemma_2} we obtain:
\begin{align*}
	\E[r_i] & \leq \frac{k|U|}{\gamma} \E \left[\eta M + \frac{\log |U|}{\eta}\right] \\
	& \leq \frac{k|U|}{\gamma} \E \left[ 6 \frac{k \log |U| \sqrt{\log (k/\delta)}}{\varepsilon}\sqrt{M+1}\right] \\
	& \leq \frac{k|U|}{\gamma} \left( 6 \frac{k \log |U| \sqrt{\log (k/\delta)}}{\varepsilon}\sqrt{\E[M]+1}\right) \tag{Jensen's inequality} \\
	& = 8 \frac{k^2 |U|\log |U| \log (k/\delta) }{\varepsilon} \sqrt{\frac{T}{\gamma}}.
\end{align*}
Using Lemma~\ref{lem:bandit_lemma_1} we obtain the $(1-1/e)$-regret bound of
\begin{align*}
	8 \frac{k^3 |U|\log |U| \log (k/\delta) }{\varepsilon} \sqrt{\frac{T}{\gamma}} + \gamma T.
\end{align*}
This is minimized at $\gamma= \Theta(1/T^{1/3})$ with a regret bound of $\mathcal{O}(T^{2/3})$ and expected space used $\Theta(T^{2/3})$. 

\section{Extension to Continuous Functions}\label{app.continuous}

In this section we prove Theorem~\ref{thm:DR-submod}. Before this, we present some preliminaries in online convex optimization.

In online convex optimization (OCO), there is compact convex set $\mathcal{X}\subseteq \R^n$ where the learner makes decisions. At time-step $t$, a convex function $f_t:\mathcal{X}\to \R$ arrives. Without observing this function, the learner has to select a point $\x_t \in \mathcal{X}$ 
based on previous functions $f_1,\ldots,f_{t-1}$. After the decision has been made, the learner receives the cost $f_t(\x_t)$ and gains oracle access to $\nabla f_t$. The learner's objective is to minimize the regret:
\[
\mathcal{R}_T=\sum_{t=1}^T f_t(\x_t) - \min_{\x\in \mathcal{X}} \sum_{t=1}^T f_t(\x).
\]
\cite{thakurta2013nearly} introduced PFTAL (Private Follow the Approximate Leader) to privately solve the OCO problem.

\begin{theorem}[\cite{thakurta2013nearly}]\label{thm:thakurta_regret}
	PFTAL is $(\varepsilon,0)$-DP and for any input stream of convex and $L$-Lipschitz functions $f_1,\ldots,f_T$ has expected regret
	\[
	\textstyle\E\left[ \mathcal{R}_T\right] \leq \mathcal{O}\left(  \frac{\sqrt{n \log^{2.5} T} \left( L + \sqrt{\frac{n \log^{2.5} T}{\varepsilon T}  } \diam \mathcal{X} \right)^2 }{\varepsilon} \sqrt{T} \right).
	\]
\end{theorem}

Similar to the Hedge algorithm, we utilize PFTAL as a black-box in Algorithm~\ref{alg:DR-submod}.

Now, we present the proof of Theorem~\ref{thm:DR-submod} in two parts, and prove each separately.

\begin{lemma}[Privacy guarantee]\label{lem:DR-privacy}
	Algorithm~\ref{alg:DR-submod} is $(\varepsilon,0)$-DP.
\end{lemma}

\begin{lemma}[Regret guarantee]\label{lem:DR-regret}
		Let $R= \sup_{\x\in \mathcal{X}} \| x\|_2$, $G$ be a bound on the gradients $\|\nabla f_t(\x_t)\|_2$, and $\beta$ be the smoothness parameter of $f_1,\ldots,f_T$.  Then Algorithm~\ref{alg:DR-submod} has $(1-1/e)$-regret
	\[\textstyle \E\left[\left( 1 - \frac{1}{e} \right)\max_{\x\in \mathcal{X}} \sum_{t=1}^T f_t(\x) - \sum_{t=1}^T f_t(\x_t)\right] = \mathcal{O}\left(  T^{3/4}\sqrt{\log^{2.5} T} \left(  \frac{\sqrt{n} \left( G + \sqrt{\frac{n}{\varepsilon T^{3/4}}} \log^{2.5} T \diam \mathcal{X}  \right)^2 }{\varepsilon}  + \beta R^2  \right)  \right).\]
\end{lemma}

\paragraph{Proof of Lemma~\ref{lem:DR-privacy}}


As with the analysis of Algorithm~\ref{alg:EXP-ALG}, we show that $(\mathcal{E}_{K-1},\ldots,\mathcal{E}_0)$ is $(\varepsilon,0)$-DP. If each $\mathcal{E}_k$ were $(\varepsilon/K,0)$-DP, then the result would immediately follow by simple composition. However, we cannot guarantee that each $\mathcal{E}_k$ is $(\varepsilon/K,0)$-DP since $\mathcal{E}_k$ obtains as input the privatized output from $\mathcal{E}_0,\ldots,\mathcal{E}_{k-1}$ in the linear function $\ell_k(\v)=\nabla f_t( \x_t^k)^\top \v$, where $\x_t^k$ is computed by $\mathcal{E}_0,\ldots,\mathcal{E}_{k-1}$, while at the same time is accessing again the function $f_t$ (and so the database) via this linear function in the gradient $\nabla f_t$. This clearly breaks the privacy that could have been gained via a simple post-processing argument and therefore and alternative method is needed. 

We do not show that each $\mathcal{E}_k$ is $(\varepsilon/K,0)$-DP but the group $(\mathcal{E}_{K-1},\ldots,\mathcal{E}_0)$ is $(\varepsilon,0)$-DP. The proof of the following lemma follows the same steps as the proof of Lemma~\ref{lem:composition_case_2}. The proof is slightly simpler since there is no $\delta$-privacy term included but it requires some care since the distributions are continuous in this case.

\begin{lemma}
	For any $i\geq 1$, the group $(\mathcal{E}_{i-1},\ldots,\mathcal{E}_0): \mathcal{F}^T \to (\mathcal{X}^T)^i$ is $i\varepsilon/K$-DP.
\end{lemma}

\begin{proof}
	We proceed by induction in $i$. The base case $i=1$ follows immediately from privacy of PFTAL in \citet{thakurta2013nearly} because $\mathcal{E}_0$ is the only algorithm that has not its distribution perturbed by any other algorithm. For the inductive step, assume the result is true for some $i\geq 1$ and let us prove it for $i+1$. 
	
	Let $\x_0^T, \ldots, \x_{i-1}^T \in \mathcal{X}^T$ and $\mathbf{X}_{i-1}=(\x_{i-1}^T,\ldots,\x_{1}^T)$. Then, for any $\x_i^T \in \mathcal{X}^T$ we have
	\begin{align*}
		\Prob(\mathcal{E}_i(F) = \x_i^T \mid (\mathcal{E}_{i-1},\ldots,\mathcal{E}_0)(F) = \mathbf{X}_{i-1} ) \leq e^{\varepsilon/K}\Prob(\mathcal{E}_i(F') = \x_i^T \mid (\mathcal{E}_{i-1},\ldots,\mathcal{E}_0)(F') = \mathbf{X}_{i-1})
	\end{align*}
	by the guarantee of PFTAL. Note that we are referring to the PMF and not the CDF of the distribution. This is because PFTAL utilizes Gaussian noise. With this, for $\mathbf{X}_{i}=(\x_i^T,\ldots,\x_0^T)$ we have,
	\begin{align*}
		&\Prob((\mathcal{E}_i,\ldots,\mathcal{E}_0)(F) = \mathbf{X}_i) \\
		& = \Prob(\mathcal{E}_i(F) = \x_i^T \mid (\mathcal{E}_{i-1},\ldots,\mathcal{E}_0)(F) = \mathbf{X}_{i-1} )\Prob((\mathcal{E}_{i-1},\ldots,\mathcal{E}_0)(F) = \mathbf{X}_{i-1}) \\
		& \leq e^{\varepsilon/K}\Prob(\mathcal{E}_i(F') = \x_i^T \mid (\mathcal{E}_{i-1},\ldots,\mathcal{E}_0)(F')= \mathbf{X}_{i-1}) \cdot e^{i\varepsilon/K} \Prob((\mathcal{E}_{i-1},\ldots,\mathcal{E}_0)(F') = \mathbf{X}_{i-1}),
	\end{align*}
	where we utilized induction and the previous inequality. This completes the proof.
\end{proof}

\paragraph{Proof of Lemma~\ref{lem:DR-regret}}

Let $G=\sup_{\substack{t=1,\ldots,T\\ \x\in \mathcal{X} }} \| \nabla f_t(\x) \|_2$. Let $r_i$ be the regret experienced by algorithm $\mathcal{E}_i$ in Algorithm~\ref{alg:DR-submod}.

The following result appears in the proof of Theorem 1 in~\cite{chen2018online}.

\begin{lemma}[\cite{chen2018online}]
	Assume $f_t$ is monotone DR-submodular and $\beta$-smooth for every $t$. Then Algorithm~\ref{alg:DR-submod} ensures
	\[
	\left( 1 - \frac{1}{e} \right)\max_{\x\in \mathcal{X}} \sum_{t=1}^T f_t(\x) - \sum_{t=1}^T f_t(\x_t) \leq \frac{1}{K}\sum_{i=0}^{K-1} r_i + \frac{\beta R^2 T}{2K}.
	\]
	where $R=\sup_{\x\in \mathcal{X}} \|\x \|_2$ and $r_i$ is the regret of algorithm $\mathcal{E}_i$. 
\end{lemma}

Using this result, we obtain
\begin{align*}
	\E\left[\left( 1 - \frac{1}{e} \right)\max_{\x\in \mathcal{X}} \sum_{t=1}^T f_t(\x) - \sum_{t=1}^T f_t(\x_t)\right] & \leq \frac{1}{K}\sum_{i=0}^{K-1} \E [r_i] + \frac{\beta R^2 }{2K} \\
	& \leq \mathcal{O}\left(  \frac{\sqrt{n \log^{2.5} T} \left( G + \sqrt{\frac{n \log^{2.5} T}{\varepsilon T/K}  } \diam \mathcal{X} \right)^2 }{\varepsilon/K} \sqrt{T}  + \frac{\beta R^2 T}{2K} \right).
\end{align*}
We can find the regret by setting $K=\left(\frac{T}{\log^{2.5} T}\right)^{1/4}$.

\end{document}